\documentclass[twoside]{article}
\usepackage{natbib}
\usepackage{subcaption}
\usepackage[accepted]{aistats2025}
\usepackage{url}            
\usepackage{booktabs}       
\usepackage{nicefrac}       
\usepackage{microtype}      
\usepackage{graphicx}
\usepackage{algorithm}
\usepackage{algorithmic}
\usepackage{subcaption}
\usepackage{amsmath,amssymb,amsfonts,bm}%
\usepackage{bbm, dsfont}

\usepackage{amssymb,amsmath,amsthm,bbm}
\usepackage{verbatim,float,url,dsfont}
\usepackage{graphicx,psfrag} 
\usepackage{algorithm,algorithmic}
\usepackage{mathtools,enumitem}
\usepackage{multirow}
\usepackage{ragged2e}
\usepackage{xr-hyper}
\usepackage{array}
\usepackage[dvipsnames]{xcolor} 
\usepackage{pifont} 
\usepackage[dvipsnames]{xcolor} 
\usepackage{etoc} 

\usepackage[colorlinks=true,citecolor=blue,urlcolor=blue,linkcolor=blue]{hyperref}

\usepackage[utf8]{inputenc} 
\usepackage[T1]{fontenc}    
\usepackage{booktabs}       
\usepackage{nicefrac}         
\usepackage{microtype}      
\usepackage{tikz} 
\usepackage{pgfplots}  
\pgfplotsset{compat=1.16}  

\ifdefined\TimesFont 
\usepackage{times} 
\fi

\ifdefined\ParSkip 
\usepackage{parskip} 
\fi

\newtheorem{theorem}{Theorem}

\theoremstyle{definition}

\newtheorem*{assumption*}{\assumptionnumber}
\providecommand{\assumptionnumber}{}


\makeatletter
\newcommand*\rel@kern[1]{\kern#1\dimexpr\macc@kerna}
\newcommand*\widebar[1]{%
  \begingroup
  \def\mathaccent##1##2{%
    \rel@kern{0.8}%
    \overline{\rel@kern{-0.8}\macc@nucleus\rel@kern{0.2}}%
    \rel@kern{-0.2}%
  }%
  \macc@depth\@ne
  \let\math@bgroup\@empty \let\math@egroup\macc@set@skewchar
  \mathsurround\z@ \frozen@everymath{\mathgroup\macc@group\relax}%
  \macc@set@skewchar\relax
  \let\mathaccentV\macc@nested@a
  \macc@nested@a\relax111{#1}%
  \endgroup
}
\makeatother

\DeclareMathOperator*{\argmin}{argmin}

\def\R{\mathbb{R}}

\def\E{\mathbb{E}}
\def\P{\mathbb{P}}

\def\cC{\mathcal{C}}

\def\cI{\mathcal{I}}

\def\cX{\mathcal{X}}
\def\cY{\mathcal{Y}}

\def\ind#1{\mathds{1}\left\{#1\right\}}

\def\lhat{\hat\lambda}

\newcommand{\revision}[1]{\textcolor{black}{#1}}

\begin{document}

%

%

\twocolumn[

\aistatstitle{Automatically Adaptive Conformal Risk Control}

\aistatsauthor{ Vincent Blot \And Anastasios N. Angelopoulos}
\aistatsaddress{ Paris-Saclay University, CNRS, \\ Laboratoire Interdisciplinaire des Sciences du Numerique, \\ 91405,
Orsay, France,  \\ Capgemini Invent France
\And  University of California, Berkeley}

\aistatsauthor{   Michael I. Jordan \And Nicolas J-B. Brunel }
\aistatsaddress{University of California, Berkeley, \\ INRIA Paris
\And ENSIIE, \\ 1 Square de la Resistance, \\ 91000, Evry-Courcouronnes, France, \\ Capgemini Invent France } ]

\begin{abstract}
  Science and technology have a growing need for effective mechanisms that ensure reliable, controlled performance from black-box machine learning algorithms.
  These performance guarantees should ideally hold \emph{conditionally on the input}---that is the performance guarantees should hold, at least approximately, no matter what the input.
  However, beyond stylized discrete groupings such as ethnicity and gender, the right notion of conditioning can be difficult to define.
  For example, in problems such as image segmentation, we want the uncertainty to reflect the intrinsic difficulty of the test sample, but this may be difficult to capture via a conditioning event.
  Building on the recent work of~\cite{gibbs2023conformal}, we propose a methodology for achieving approximate conditional control of statistical risks---the expected value of loss functions---by adapting to the difficulty of test samples.
  Our framework goes beyond traditional conditional risk control based on user-provided conditioning events to the algorithmic, data-driven determination of appropriate function classes for conditioning.
  We apply this framework to various regression and segmentation tasks, enabling finer-grained control over model performance and demonstrating that by continuously monitoring and adjusting these parameters, we can achieve superior precision compared to conventional risk-control methods.
\end{abstract}

\section{INTRODUCTION}
\label{sec:introduction}

Conformal prediction~\citep{vovk2005algorithmic} has emerged over the last several years as a promising solution for quantifying uncertainty in black-box machine learning models via prediction sets.
Conformal risk control~\citep{angelopoulos2024conformal} extends the conformal methodology to high-dimensional and structured data tasks, such as image segmentation, where the standard notion of coverage does not naturally apply.
These techniques are especially attractive due to their model- and distribution-agnostic nature; their validity does not rely on any assumptions about the model class at hand or the particular data distribution~\citep{vovk2005algorithmic}.
A limitation of classical conformal techniques, however, is its inability to provide conditional guarantees.
Thus, the quality of the uncertainty quantification can depend on the input covariates and degrade in some parts of the input space, especially where data is scarce, even if the average quality of uncertainty quantification is controlled.

While conditional guarantees are impossible in full generality for any algorithm~\citep{vovk2012conditional}, recent progress has been made on tractable relaxations of conditional coverage.
In particular, \cite{gibbs2023conformal} introduce an extension of conformal prediction that gives exact coverage conditionally on overlapping groups, and additionally, can provide a relaxed form of conditional coverage against certain covariate shifts parameterized by a user-chosen function class $\mathcal{F}$.
For a non-expert user, however, specifying $\mathcal{F}$ can be hard, and in many prediction tasks, there is no clear choice even for the expert user.
Indeed, there are many tasks for which users do not have any conditioning events in mind, but rather, simply want their uncertainty to adapt automatically to the difficulty of the test sample.

In this paper, we introduce a procedure--- \emph{automatically adaptive conformal risk control (AA-CRC)}---that involves two innovations: (1) it obviates the need to pick a function class $\mathcal{F}$ in~\cite{gibbs2023conformal} by providing a theoretically motivated algorithm for selection of $\mathcal{F}$, and (2) it extends the arguments of~\cite{gibbs2023conformal} for conformal prediction to conformal risk control.
We also extend~\cite{gibbs2023conformal} to handle label-conditional coverage.
Auto-adaptive CRC thus adapts more carefully to the difficulty of the input sample, and the resulting uncertainty better reflects the true errors of the model.
As an important practical side effect, AA-CRC generally has substantially better statistical power than conformal prediction or conformal risk control alone.
For an example of this improved performance, see Figure~\ref{fig:teaser}. The code is available at \url{https://github.com/vincentblot28/multiaccurate-cp} and all datasets used for the experiments are open source.

\begin{figure*}[t]
    \centering
    \includegraphics[width=.8\textwidth]{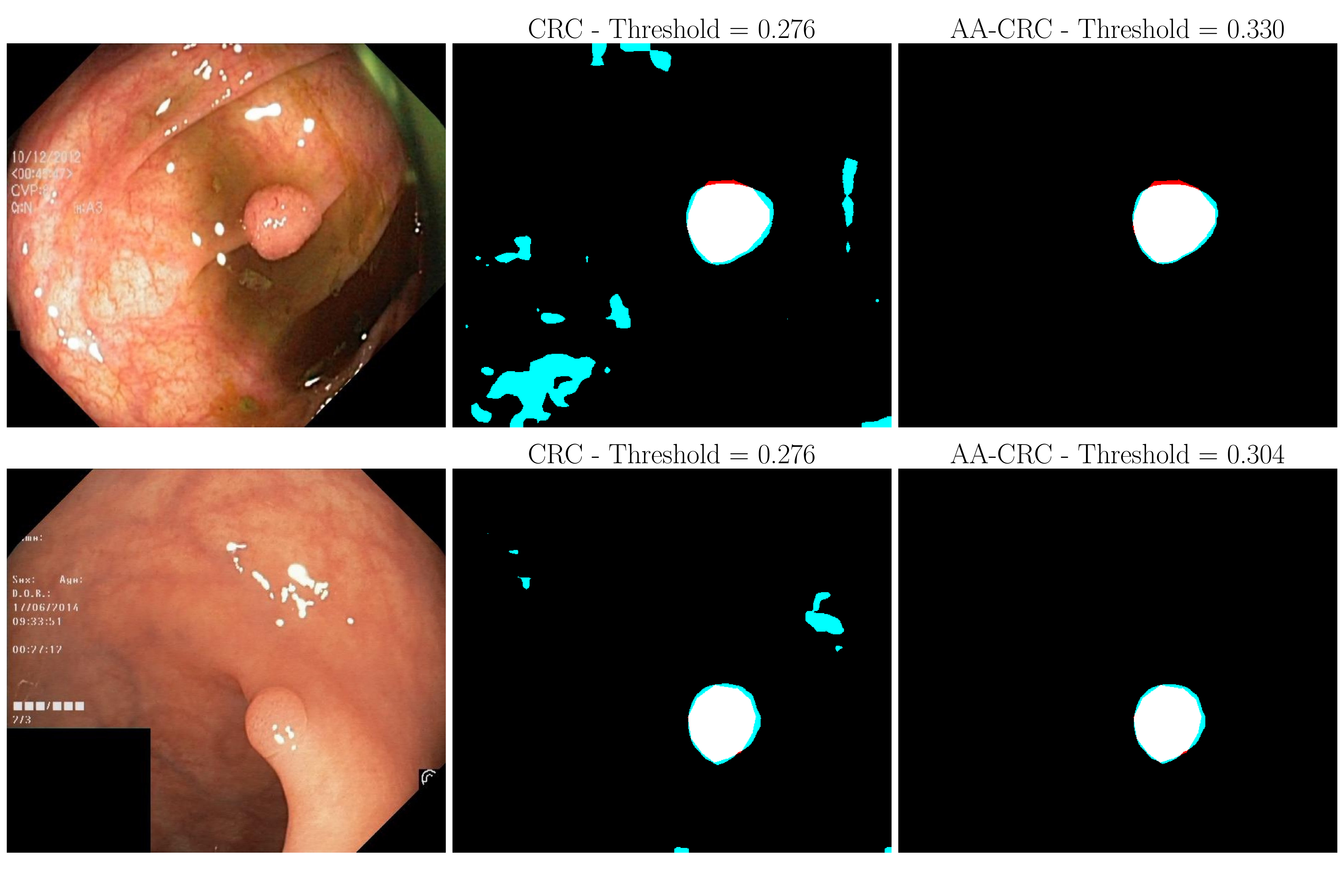}
    \caption{Example of polyp segmentations with conformal risk control (CRC) and our methodology (AA-CRC), where the true positive pixels are in white and the false positives in blue. To guarantee the recall on the image, our method outputs a threshold equal to 0.304 and 0.330 while the constant threshold of the CRC methodology is 0.276. This difference implies a higher precision for our methodology.}
    \label{fig:teaser}
\end{figure*}

\subsection{Problem statement}
Consider a dataset of exchangeable feature-label pairs, $(X_1, Y_1), \ldots, (X_{n+1},Y_{n+1}) \in \cX \times \cY$, where the last label $Y_{n+1}$ is our target (an unknown quantity we want to predict).
Consider a set-valued predictor $\cC_{\lambda}(x)$, indexed by $\lambda$.
We would like this set to have a low risk---or expected loss---as measured by a loss function $\ell(\cC_{\lambda}(x), y)$.
An example of a loss function is the false negative rate in multilabel classification: $\ell(\cC_{\lambda}(x), y) = \frac{|y \setminus \cC_{\lambda}(x)| }{|y|}$.
Conformal risk control, as defined in~\cite{angelopoulos2024conformal}, offers guarantees of the form
\begin{equation}
    \label{eq:standard-crc}
    \E\left[ \ell(\cC_{\hat{\lambda}}(X_{n+1}), Y_{n+1})\right] \leq \alpha,
\end{equation}
provided that $\ell$ is monotone nonincreasing when viewed as a function of $\lambda$.
The goal of our work is to extend the above guarantee analogously to (2.3) of~\cite{gibbs2023conformal}:
\begin{equation}
    \label{eq:multiaccuracy}
    \E\left[ \frac{\lambda(X_{n+1})}{\E[\lambda(X_{n+1})]}\left(\ell(\cC_{\hat{\lambda}(X_{n+1})}(X_{n+1}), Y_{n+1}) - \alpha \right) \right] \leq 0,
\end{equation}
for any non-negative $\lambda \in \Lambda$, where $\Lambda$ is some class of functions that map $\cX$ to $\R$.
Following~\cite{gibbs2023conformal}, the choice of function class $\Lambda$ will determine the type of multiaccuracy guarantees we are able to achieve.

To better understand the guarantee in~\eqref{eq:multiaccuracy}, we give several examples.
\begin{enumerate}
    \item When $\Lambda = \{ x \mapsto 1 \}$, we recover standard, marginal conformal risk control, and~\eqref{eq:multiaccuracy} becomes equivalent to~\eqref{eq:standard-crc}.
    \item Let $\Phi$ map $\cX$ to a $d$-dimensional binary vector. One can think of $\Phi(x)$ as a vector of group indicators. When $\Lambda = \{ \Phi(x)^\top \theta : \theta \in \R^d \} $, we obtain group-conditional conformal risk control with overlapping groups:
    \begin{equation}
        \begin{split}
            \E\left[ \ell(\cC_{\hat{\lambda}(X_{n+1})}(X_{n+1}), Y_{n+1}) \mid \Phi(X_{n+1}) = j \right] \leq \alpha,\\
            \qquad \forall j \in [d].
        \end{split}
    \end{equation}
    \item Let $\Phi$ be a $d$-dimensional neural network embedding of $X$, and let $\Lambda = \{ \Phi(x)^\top \theta : \theta \in \R^d \} $. Then our method provides a risk control guarantee over a set of covariate shifts:
    \begin{equation}
        \E_{\lambda}\left[ \ell(\cC_{\hat{\lambda}(X_{n+1})}(X_{n+1}), Y_{n+1}) \right] \leq \alpha, \qquad \forall \lambda \in \Lambda,
    \end{equation}
    where $\E_\lambda$ is defined as the expected value when \revision{tilted by $\frac{\lambda(X_{n+1})}{\E[\lambda(X_{n+1})]}$, as in~\eqref{eq:multiaccuracy}.}
    In other words, our guarantee is robust to all covariate shifts that are linear in embedding space.
    \revision{It is by learning $\Phi$ to provide features that result in a high-accuracy linear predictor of the error on $X_{n+1}$ that we achieve our ``automatic'' notion of conditional validity.}
\end{enumerate}

\subsection*{Related work}
We study the topic of conformal prediction~\citep{vovk2005algorithmic} under relaxed notions of conditional coverage.
There is a large volume of work on conformal prediction and conditional coverage, most notably the foundational works of~\cite{vovk2012conditional}, ~\cite{barber2021limits}.
Also foundational are the works of~\cite{jung2021moment,jung2022batch,bastani2022practical}, which introduce the relationship between conformal prediction, quantile regression, and conditional coverage on discrete but overlapping groups.
These papers introduce the idea in a high-probability, split-conformal sense; it is developed by~\cite{gibbs2023conformal} in a full-conformal, marginal sense over possibly continuous groups, as in the present manuscript.
We remark that the guarantee in~\eqref{eq:multiaccuracy} resembles the multi-accuracy guarantee in (1) of~\cite{kim2019multiaccuracy}, although the mathematical tools we use are unrelated, as far as we know.
The closest ancestors of our work are~\cite{gibbs2023conformal} and~\cite{angelopoulos2024conformal}.
Our paper combines the guarantees from these two lines of work.
As we will soon see, combining these approaches is not trivial, and stems from a new reframing of conformal risk control as the solution to an implicit optimization problem.
An additional novelty as compared to~\cite{gibbs2023conformal} is suggesting an automatic algorithm for selecting the function class $\Lambda$ in order to achieve better general purpose conditional performance---this is critical, as there is no clear choice of $\Lambda$ in many practical problems, so this improves upon the practical value of~\cite{gibbs2023conformal} even in the standard conformal setup.
Along the same lines, we extend the guarantee of~\cite{gibbs2023conformal} to handle label-conditional coverage, and more generally allow $\mathcal{F}$ to be a class of mappings that depend on both the covariate and the label.
Finally, we refer the reader to the related concurrent work of~\cite{zhang2024fair} on fair risk control; their problem setting is similar to ours, while their algorithms and guarantees are different but complementary to ours.
\section{THEORY}

\subsection{Background}
We begin by reinterpreting conformal prediction in the language of the first-order optimality conditions of standard quantile regression~\citep{koenker1978regression}.
Let $D^y = ( (X_1,Y_1), \ldots, (X_n,Y_n), (X_{n+1},y))$ denote a putative dataset where the $(n+1)$st label is replaced with the putative label $y$.
Let $s : \cX \times \cY \to \R$ be a conformal score.
Also, let $\cC_{\lambda}(x) = \{ y : s(x,y) \leq \lambda \}$ be the split-conformal prediction set formed \revision{using threshold} $\lambda$.
Finally, let \revision{$\rho(u) = (1 - \alpha)u \ind{u \geq 0} -\alpha u \ind{u < 0}$} be the pinball loss.

The first step to understanding our approach is to reframe conformal prediction as a form of intercept-only quantile regression.
Let
\revision{
\begin{equation}
    J(\lambda, D^y) = \sum_{i=1}^{n} \rho(s(X_i,Y_i)-\lambda) + \rho(s(X_{n+1},y)-\lambda),
\end{equation}}
and let $\hat\lambda^y = \argmin_{\lambda \in \Lambda} J(\lambda, D^y)$. 
As in~\cite{gibbs2023conformal}, it is straightforward to verify that the standard split-conformal prediction set is formed as
\begin{equation}
    \cC(X_{n+1}) = \{ y : s(X_{n+1},y) \leq \hat\lambda^y \}.
\end{equation}

In addition to the procedure being equivalent to a form of quantile regression, the coverage guarantee can be rephrased in the language of the first-order conditions of quantile regression as well.
As in any optimization problem, the first-order optimality condition states that $0 \in \partial J(\hat\lambda^y, D^y)$.
Accordingly, for all $i \in [n+1]$, we define \revision{$g_i$ to be the subgradient function of $\lambda \mapsto \rho(s(X_i,Y_i)-\lambda)$} characterized as follows:
\begin{enumerate}
    \item $g_i(\lambda) = \ind{Y_i \notin \cC_{\lambda}(X_i)} - \alpha $ if $\lambda \neq s(X_i,Y_i)$.
    \item $g_{n+1}(\lambda) = \ind{y \notin \cC_{\lambda}(X_{n+1})} - \alpha$ if $\lambda \neq s(X_i, y)$.
    \item $g_i(\lambda) \in [-\alpha, 1-\alpha]$.
    \item \revision{$\sum\limits_{i=1}^{n+1} g_i(\hat\lambda^y) = 0$}.
\end{enumerate}
Using these constraints, and defining $\cI = \{ i : \lhat = s(X_i,y) \}$, we have:
\begin{equation}
    \begin{aligned}
    \MoveEqLeft{\sum\limits_{i=1}^{n+1} g_i(\hat\lambda^{Y_{n+1}}) = 0} \\
    \Longleftrightarrow & \frac{1}{n+1}\sum\limits_{i=1}^{n+1} \ind{Y_i \notin \cC_{\lhat}(X_i)} - \alpha = \\
    & \frac{1}{n+1}\sum\limits_{i \in \cI} (\ind{Y_i \notin \cC_{\lhat}(X_i)} - \alpha - g_i(\lhat^{Y_{n+1}})).
\end{aligned}
\end{equation}

Note that for all $i \in \cI$, $\ind{Y_i \notin \cC_{\lhat}(X_i)} = 0$, and $g_i(\lhat^{Y_{n+1}}) \geq -\alpha$.
Thus, the right-hand side of the displayed equation is nonpositive, implying that $\frac{1}{n+1}\sum\limits_{i=1}^{n+1} \ind{Y_i \notin \cC_{\lhat}(X_i)} \leq \alpha$.
The standard conformal argument completes the proof: 
\begin{equation}
    \begin{split}
        \P(Y_{n+1} \notin \cC(X_{n+1})) = \P(Y_{n+1} \notin \cC_{\lhat^{Y_{n+1}}}(X_{n+1})) = \\
        \E\left[ \frac{1}{n+1}\sum\limits_{i=1}^{n+1} \ind{Y_i \notin \cC_{\hat\lambda^{Y_{n+1}}}(X_i)}\right] \leq \alpha.
    \end{split}
\end{equation}

The work of~\cite{gibbs2023conformal} extends the above argument beyond intercept-only quantile regression; roughly, the idea is to define a vector space of functions $\Lambda$ whose elements map $\cX$ to $\R$, and then repeat the argument above.
We omit the details here, since the argument will be clear from the proof of our main theorem.

\subsection{Main results}
\label{sec:main}

We now build up to our main result, which is analogous to Theorem 3 of~\cite{gibbs2023conformal}.
The main difference is that we do not provide a conditional coverage guarantee, but rather, a conditional risk control guarantee.
Furthermore, we handle function classes that depend on both the input and output: $\Lambda = \{ (x,y) \mapsto \lambda(x,y) \}$.
This allows the methodology to capture both group-conditional and label-conditional coverage.

Consider a nested family of sets, $\cC_{u}(x)$, indexed by $u$.
Let $\ell : \cY \times 2^{\cY} \to [0,1]$ be a left-continuous and monotone nonincreasing loss function:
\begin{equation}
    \cC_1 \subseteq \cC_2 \implies \ell(y, \cC_1) \geq \ell(y, \cC_2).
\end{equation}
For convenience we will abuse notation to write $\ell(x, y, u) = \ell(y, \cC_u(x))$.
We also define a related indefinite integral (or antiderivative), $I : \cX \times \cY \times \R \to \R$, as
\begin{equation}
    I(x,y,u) = \int (\ell(x,y,u') - \alpha) du'\Big\vert_u.
\end{equation}
\revision{In the above equation, $u'$ is a dummy variable used for the purpose of calculating the antiderivative, and we evaluate the antiderivative at the value $u$.}
Because $\ell$ is a monotone loss function, we are guaranteed that $I$ is a quasiconvex function; this will pose some interesting challenges for forming the prediction set, as we will soon see.
As an additional challenge, unlike the case of the pinball loss, this indefinite integral can not be computed analytically in general.

We now define the following functions, analogously to the previous section: 
\begin{equation}
    \begin{split}
        J^y(\lambda) = \frac{1}{n+1} \sum_{i=1}^{n} I(X_i, Y_i, \lambda(X_i,Y_i)) +\\
        \frac{1}{n+1} I(X_{n+1}, y, \lambda(X_{n+1},y)) + \mathcal{R}(\lambda),
    \end{split}
\end{equation}
where $\mathcal{R} : \Lambda \to \R$ is a regularizer,
\begin{equation}
    \lhat^y = \argmin_{\lambda \in \Lambda} J^y(\lambda),
\end{equation}
and
\begin{equation}
    \cC(x) = \cC_{\sup_{y \in \cY}\lhat^y(x)}(x).
\end{equation}
Then the set $\cC(X_{n+1})$ has the following guarantee.
\begin{theorem}
    \label{thm:main-validity}
    Consider a vector space $\Lambda$ equipped with the standard addition operation, and assume that for all $\lambda, \lambda' \in \Lambda$, the derivative $\epsilon \mapsto \mathcal{R}(\lambda + \epsilon \lambda')$ exists.
    If $\lambda$ is nonnegative and $\E[\lambda(X_{n+1},Y_{n+1})] > 0$, then 
    \begin{equation}
        \begin{split}
            \E_{\lambda}[\ell(Y_{n+1}, \cC(X_{n+1}))] \leq \\
            \alpha - \frac{1}{\E[\lambda(X_{n+1},Y_{n+1})]} \E\left[\frac{d}{d\epsilon} \mathcal{R}(\lhat^{Y_{n+1}} + \epsilon \lambda) \Big\vert_{\epsilon = 0} \right].
        \end{split}
    \end{equation}
\end{theorem}
\begin{proof}
    Pick any $\lambda \in \Lambda$ and $\epsilon \in [0,1]$.
    For all $y \in \cY$, because $\lhat^y$ is a minimizer of $J^y$ and $\Lambda$, we have that the first-order optimality condition is satisfied.
    Thus, for $Y_{n+1}$,
    \begin{equation}
        0 \in \partial_{\epsilon} J^{Y_{n+1}}(\lhat^{Y_{n+1}} + \epsilon \lambda) \bigg|_{\epsilon = 0}.
    \end{equation}
    Now we will define any subgradients $g_i(\lambda)$ that satisfy a similar list of conditions as we previously defined in the proof of conformal prediction.
    We start by defining some useful functions. Let $L_i(\lambda) = \lambda(X_i,Y_i)(\ell(X_i,Y_i,\lhat^{Y_{n+1}}(X_i,Y_i))-\alpha)$, $U_i(\lambda) = \lim_{\epsilon \to 0^+}\lambda(X_i,Y_i)(\ell(X_i,Y_i,(\lhat^{Y_{n+1}} + \epsilon\lambda)(X_i,Y_i))-\alpha)$, and $r(\lambda) := \frac{d}{d\epsilon} \mathcal{R}(\lhat^{Y_{n+1}} + \epsilon \lambda) \Big\vert_{\epsilon = 0}$, for $i \in [n+1]$.
    Importantly, $U_i \geq L_i$ deterministically, since $(\lhat^{Y_{n+1}} + \epsilon\lambda)(X_i,Y_i) = \lhat^{Y_{n+1}}(X_i,Y_i) + \epsilon\lambda(X_i,Y_i)$ and $\lambda(X_i,Y_i) \geq 0$.
    Returning to the conditions for the subgradients, we pick any subgradients $g_1, \ldots, g_{n+1}$ satisfying
    
    \begin{enumerate}
        \item $g_i(\lambda) \in [L_i(\lambda), U_i(\lambda)]$.
        \item $\frac{1}{n+1}\sum\limits_{i=1}^{n+1} g_i(\lhat^{Y_{n+1}}) + r(\lambda) = 0$.
    \end{enumerate}
    Let 
    \begin{equation}
        \mathcal{I}_1 = \left\{ i \in [n+1] : L_i(\lambda) \neq U_i(\lambda) \right\}.
    \end{equation}
    Then, we can write
    \begin{align}
        \MoveEqLeft{\frac{1}{n+1}\sum\limits_{i=1}^{n+1} g_i(\hat\lambda^{Y_{n+1}}) + r(\lambda) = 0} \\
        \Longleftrightarrow & \frac{1}{n+1}\sum\limits_{i=1}^{n+1}\lambda(X_i,Y_i)(\ell(X_i,Y_i,\lhat^{Y_{n+1}}(X_i,Y_i)) - \alpha) \\
        & \qquad \qquad = \frac{1}{n+1}\sum\limits_{i \in \cI} \Big(L_i - g_i(\lhat^{Y_{n+1}})\Big) - r(\lambda).
    \end{align}
    But for all $i \in \cI$, $g_i(\lhat^{Y_{n+1}}) \geq L_i(\lambda)$.
    Thus, the right-hand side of the displayed equation is no greater than $-r(\lambda)$, implying that $\frac{1}{n+1}\sum\limits_{i=1}^{n+1} \ell(X_i,Y_i,\lhat^{Y_{n+1}}(X_i,Y_i)) -\alpha \leq - r(\lambda)$.
    
    Now we apply our standard exchangeability arguments.
    By exchangeability and the symmetry of $\lhat^{Y_{n+1}}$, we have that
    \begin{align}
        \MoveEqLeft{\E[\lambda(X_i,Y_i)(\ell(X_i,Y_i,\lhat^{Y_{n+1}}(X_i,Y_i)) - \alpha)]} \\
        &\hspace{-3em} = \E\left[\frac{1}{n+1}\sum\limits_{i=1}^{n+1} \lambda(X_i,Y_i)(\ell(X_i,Y_i,\lhat^{Y_{n+1}}(X_i,Y_i)) - \alpha) \right] \\
        & \leq -\E[r(\lambda)].
    \end{align}
    Thus, rearranging terms, $\E_\lambda[\ell(X_i,Y_i,\lhat^{Y_{n+1}}(X_i,Y_i))] \leq \alpha - \frac{1}{\E[\lambda(X_i,Y_i)]}\E[r(\lambda)]$.

    For the final conclusion, $\cC(X_{n+1}) \supseteq \cC_{\lhat^{Y_{n+1}}}$, by definition we have that 
    \begin{equation}
        \begin{split}
            \E_{\lambda}[\ell(Y_{n+1}, \cC(X_{n+1}))] \leq \E_{\lambda}[\ell(Y_{n+1}, \cC_{\lhat^{Y_{n+1}}}(X_{n+1}))] \\
        \leq \alpha - \frac{1}{\E[\lambda(X_i,Y_i)]}\E[r(\lambda)].
        \end{split}
    \end{equation}
        
\end{proof}

\subsection{Efficient computation of $\lhat$}

The procedure we have outlined thus far is analogous to full conformal risk control (see~\cite{angelopoulosnote}), in that we must loop over all values of $y \in \cY$ to calculate $\lhat$.
We have avoided including the model retraining as part of this procedure, but regardless, it may be impossible or infeasible to loop through $\cY$.

However, this can be avoided.
In particular, assume for all $\lambda \in \Lambda$ and all $(x,y)$ that $\lambda(x,y) \leq \nu(x)$.
\revision{For example, if $\Lambda$ is a class of bounded functions, then $\nu(x)$ can be set to the bound.}
Another special case is when $\lambda(x,y)$ does not depend on $y$, in which case $\nu$ exists trivially.
The following optimization problem also provides risk control, but does not require looping through $y \in \cY$:

\begin{equation}
    \begin{split}
    \tilde{\lambda} = \argmin_{\lambda \in \Lambda} \tilde{J}(\lambda) = \frac{1}{n+1} \sum_{i=1}^{n} I(X_i, Y_i, \lambda(X_i)) + \\ \frac{1-\alpha}{(n+1)}\lambda(X_{n+1}) + \mathcal{R}(\lambda).
    \end{split}
\end{equation}
To see why this algorithm provides risk control, assume for convenience that $\ell$ is continuous in its last argument.
Then, the first-order optimality of $\tilde\lambda$ implies that for all $\lambda$,

\begin{align*}
    & \frac{d}{d\epsilon} \tilde{J}(\tilde\lambda + \epsilon\lambda) \Big|_{\epsilon=0} = 0 \\
    \implies & \frac{d}{d\epsilon} \Bigg( 
        \frac{1}{n+1} \sum_{i=1}^{n} I(X_i, Y_i, \tilde\lambda(X_i)+\epsilon\lambda(X_i))  \\
        & \quad + \frac{1-\alpha}{(n+1)}(\tilde\lambda(X_{n+1}) + \epsilon\lambda(X_{n+1})) \\
        & \quad + \mathcal{R}(\tilde\lambda + \epsilon\lambda)
    \Bigg) \Bigg|_{\epsilon=0} = 0 \\
    \implies & \frac{1}{n+1} \sum\limits_{i=1}^{n} \lambda(X_i)(\ell(X_i,Y_i,\tilde\lambda(X_i)) -\alpha) \\
    & \quad + \frac{1-\alpha}{n+1} \lambda(X_{n+1})  \\
    & \quad = -\frac{\partial}{\partial \epsilon} \mathcal{R}(\tilde\lambda + \epsilon\lambda) \Big|_{\epsilon = 0} \\
    \implies & \frac{1}{n+1} \sum\limits_{i=1}^{n+1} \lambda(X_i)(\ell(X_i,Y_i,\tilde\lambda(X_i)) -\alpha) \\
    & \quad \leq -\frac{\partial}{\partial \epsilon} \mathcal{R}(\tilde\lambda + \epsilon\lambda) \Big|_{\epsilon = 0}.
\end{align*}

from which we can then continue on with the same exchangeability arguments in Theorem 3 to prove a risk-control bound.

We make some final observations about solving this optimization problem. 
In the absence of regularization, solving an optimization problem over $J(\lambda)$ is a quasiconvex optimization problem, and standard first-order methods can get stuck in saddle points or local minima.
However, a saddle point or local minimum is fine from the purpose of risk control---our analyses rely only on local first-order optimality conditions.
For maximum performance, it is best to escape the saddle points to find the global minimum; noisy gradient descent has been shown to be an effective method for this purpose~\citep{jin2017escape}.
All that said, in our experiments, we have never encountered a problem with the standard SciPy automatic solvers, such as \texttt{scipy.optimize.minimize}.
\section{RESULTS}

We have not yet discussed one of our main contributions: how do we pick $\Lambda$ automatically?
Our answer is straightforward: the preceding sections have shown that we should think of $\lhat(X_i,Y_i)$ as an error-prediction algorithm, much like the scorecaster of~\cite{angelopoulos2023conformalpid}.
We use this perspective to parameterize the function class to yield the best predictor possible.

For semantic segmentation tasks, we use a standard deep-learning approach in which we train a convolutional neural network that predicts the highest threshold (on the softmax of pixels) such that the risk (on this single image) is lower than $\alpha$.
We then slice off the last fully connected layer, and the resulting feature extractor becomes our $\Phi(x)$.
The class of functions $\Lambda = \{x \mapsto \Phi(x)^{\top}\theta \vert \theta \in \mathbb{R}^d\}$ is defined as the space of linear functions of this embedding.
This procedure is presented in Figure~\ref{fig:embedding_procedure}. 
This essentially amounts to a rigorous method for fine-tuning a fully-connected layer on a pretrained network backbone to provide risk-controlled estimates.

For tabular regression tasks, where neural networks are not the tool of choice~\citep{bad_nn}, we create an embedding with a Random Forest (RF)~\citep{breiman2001random}, building on the work of~\cite{amoukou2023adaptive} for computing adaptive predictive intervals.
Our Algorithm~\ref{alg:rf_train_calib} is analogous to theirs.
The main idea is to train a RF model to learn quantiles of the error distribution of the base model, and then to consider each leaf a group.
This procedure assigns to each observation as many groups as there are trees in the RF.
We then set $\Phi(x)$ to be the vector of group indicators, and $\Lambda$ to be the space of linear functions of $\Phi(x)$.

\subsection{Regression task}

Our first example is a simple simulation in the context of prediction intervals.
This experiment is primarily meant to visually showcase the automatic selection of groups.
In this setting, let $f$ be any regression model that takes as input $x \in \mathbb{R}$ and predicts $y \in \mathbb{R}$. 
The goal is to control the coverage of our prediction intervals; hence, our loss function will be defined as follows:
\begin{equation}
    \ell(x, y, \lambda) = \mathds{1}\{y \in \mathcal{C}_{\lambda(x,y)}(x)\}, \text{ where } \mathcal{C}_{u}(x) = [\hat{f}(x) \pm u]
\end{equation}
for all $u \in \R$.


\begin{algorithm}
\scriptsize
\caption{Random Forest Training and inference for automatic group creation}
\label{alg:rf_train_calib}
\begin{algorithmic}[1]
\REQUIRE $\mathcal{D}_{res} = \{(X_1, |y_1 - \hat{y}_1|), \ldots, (X_N, |y_N - \hat{y}_N|)\}$, $\mathcal{D}_{cal} = \{(X_1, y_1), \ldots, (X_m, y_m)\}$
\STATE Initialize RF $\leftarrow$ \textsc{RandomForest}
\STATE \textbf{Train} RF on $\mathcal{D}_{res}$
\FORALL{element $x$ in $\mathcal{D}_{cal}$}
    \STATE $G \gets [0]^{|\text{RF}|}$ \COMMENT{\(|RF|\) is the number of leafs in the forest}
    \FORALL{tree $T \in \text{RF}$}
        \FORALL{leaf $L \in T$}
            \IF{$x \in L$}
                \STATE $G[L] \gets 1$
            \ENDIF
        \ENDFOR
    \ENDFOR
\ENDFOR
\end{algorithmic}
\end{algorithm}

We use a simulated dataset from~\cite{cqr}. We used 2000 points for training, 1000 for the residual RF training, and 9000 for calibration and 5000 test points. Results are reported in Figure~\ref{fig:cqr}. The obtained marginal coverage is 0.897 with a target coverage  $1 - \alpha = 0.9$ and the coverage of each group varies between 0.886 and 0.913, which is within the expected fluctuations for a test set of this size. 

\begin{figure}[!ht]
    \centering
    \begin{subfigure}[b]{0.4\textwidth}
        \centering
        \includegraphics[width=\textwidth]{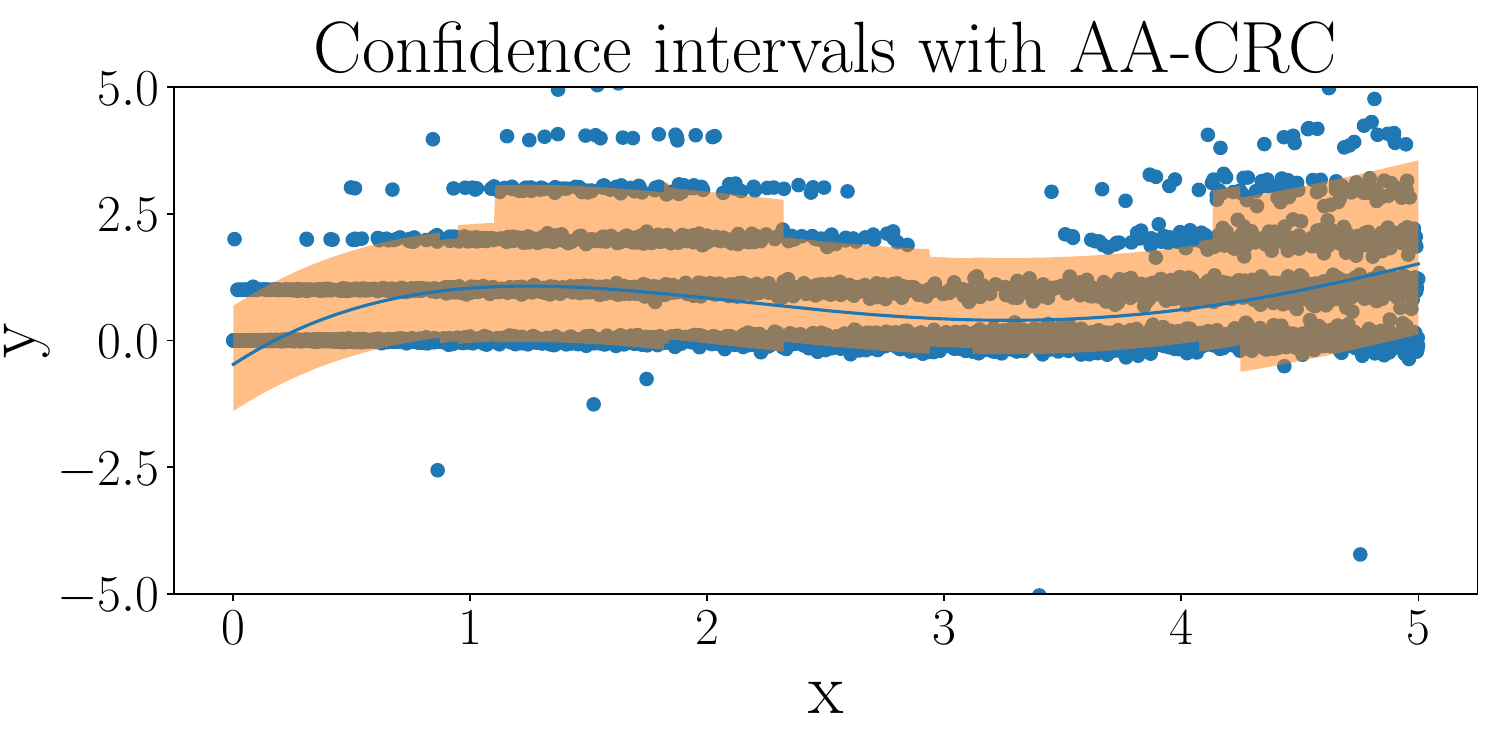}
    \end{subfigure}
    \hfill
    \begin{subfigure}[b]{0.4\textwidth}
        \centering
        \includegraphics[width=\textwidth]{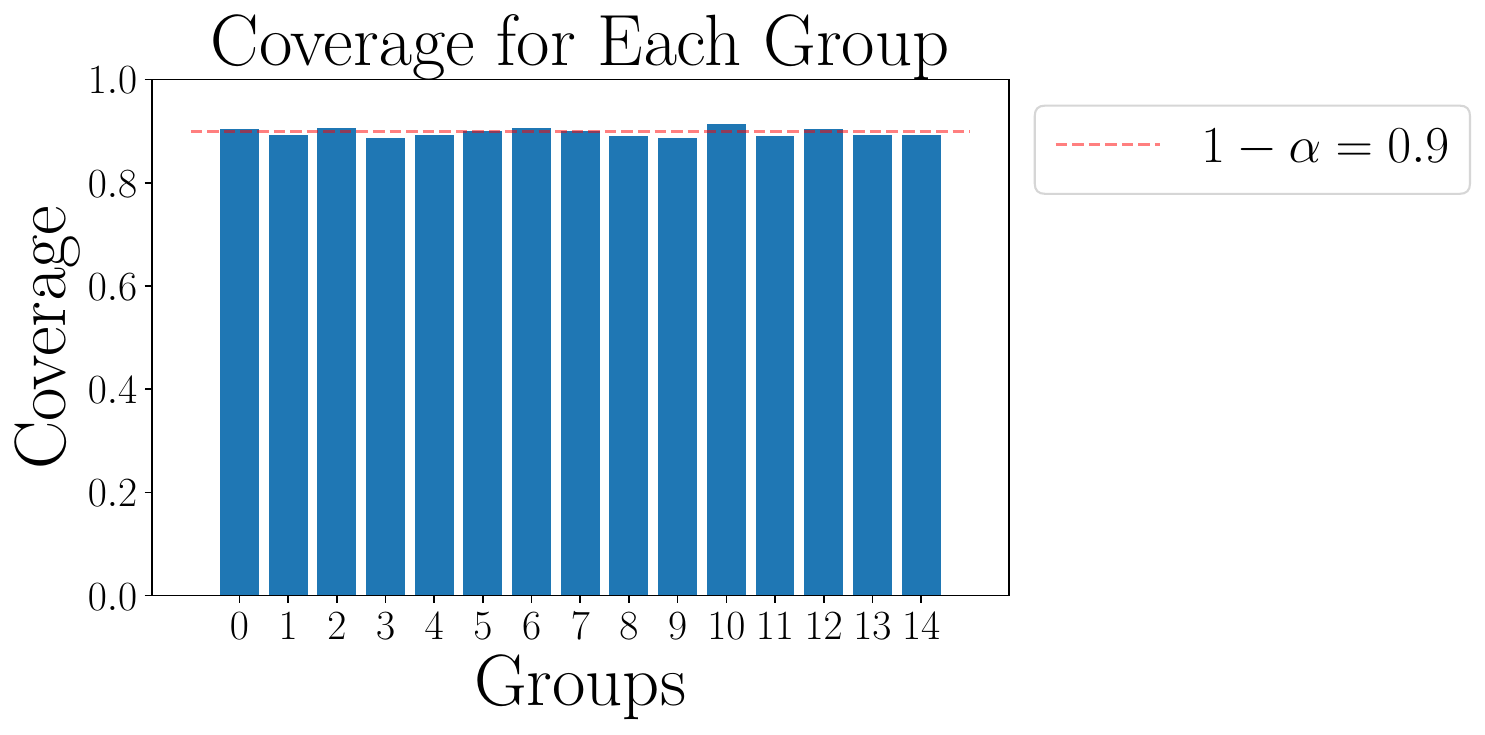}
    \end{subfigure}
    \caption{\textbf{Top figure.} The blue curve is the model prediction, blue dots are test data points, and prediction intervals are shown in orange. \revision{The dotted lines represent the CRC prediction intervals which have constant width}. \textbf{Bottom figure.} It shows the within-group coverage for each of the adaptively selected groups. The red line is the target coverage level. 
    The coverage is almost exact for all groups \revision{for AA-CRC while almost all groups are either undercovered or overcovered this the standard CRC method}.}
    \label{fig:cqr}
\end{figure}

\begin{figure*}[!ht]
    \centering
    \includegraphics[width=.7\textwidth]{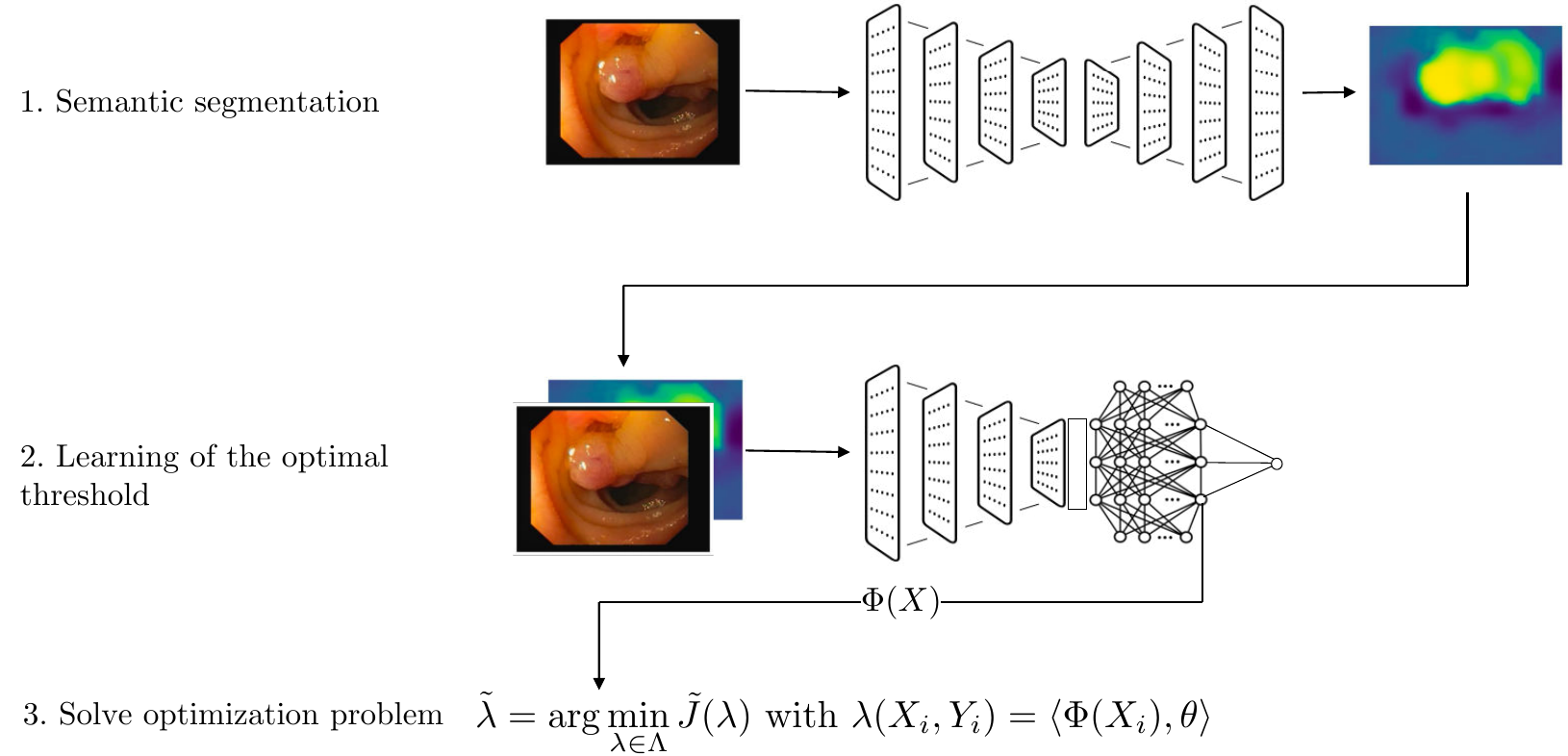}
    \caption{Procedure to create the embedding of the images. The \textbf{first step} is the training of the segmentation model on the $\mathcal{D}_{train}$ dataset. The \textbf{second step} is the learning of the embedding based on the segmentation output on the $\mathcal{D}_{res}$ dataset. The \textbf{third step} is the solving of the optimization procedure and the $\mathcal{D}_{cal}$ dataset.}
    \label{fig:embedding_procedure}
\end{figure*}

\subsection{Semantic segmentation}

In this setting, let $f$ be any semantic segmentation model which takes as input $x \in \mathbb{R}^{d_1 \times d_2 \times c}$ and predicts sigmoids $f(x) \in [0, 1]^{d_1 \times d_2}$. 
Our target is a binary segmentation mask in $\cY = \{0,1\}^{d_1 \times d_2}$, and for any $y \in \cY$, we abuse notation and refer to $|y| = \mathds{1}^{\top}y\mathds{1}$ as the sum of all the pixels.
The objective here is to control the recall of the segmentation model.
In particular, we index our final segmentation with threshold $u \in [0,1]$ as $\cC_{u}(x) \in \cY$, and $\cC_{u}(x)_{i,j} = \ind{f(x)_{i,j} \geq u}$.
With this in hand, the loss function $\ell$ is defined as follows:

\begin{equation}
    \ell(x, y, \lambda) = 1 - \frac{y \cap \mathcal{C}_{\lambda(x, y)}(x)}{|y|}.
\end{equation}



\paragraph{Polyp segmentation dataset.}
For this experiment, we used a PraNet \citep{fan2020pranet} model for the semantic segmentation and a ResNet-50 \citep{resnet} for the embedding learning.
We chose an embedding size of 1024.
Both models were trained on Kvasir-SEG \citep{ksavir} and CVC-ClinicDB \citep{cvc}, and the calibration and testing were performed on the CVC-300 \citep{cvc300}, CVC-ClinicDB, CVC-ColonDB \citep{colondb}, ETIS-LaribPolypDB \citep{etis}, and Kvasir datasets. In total, 1450 images were used for the training and embedding learning, and 798 images were used for calibration and testing. Results are reported in Figure~\ref{fig:main_polyp} with $\alpha = 0.1$. 


The mean and standard deviation of the recall over 100 random splits are 0.906 and 0.021 respectively. The average precision of the AA-CRC method is 0.457 versus 0.395 for standard CRC, showing a significant improvement of the precision while guaranteeing the same level of recall. \revision{The reason we showed the improvement in precision is that the improved conditional performance essentially allows us to improve the overall performance of the method: that is, subject to the same type-1 error $\alpha$, the adaptive procedure gets better type-2 error. This does imply better conditional performance, since more correct segmentations are being made, allowing us to allocate the error budget more efficiently.}

\begin{figure}[!ht]
    \centering
    \begin{subfigure}[b]{0.4\textwidth}
        \centering
        \includegraphics[width=\textwidth]{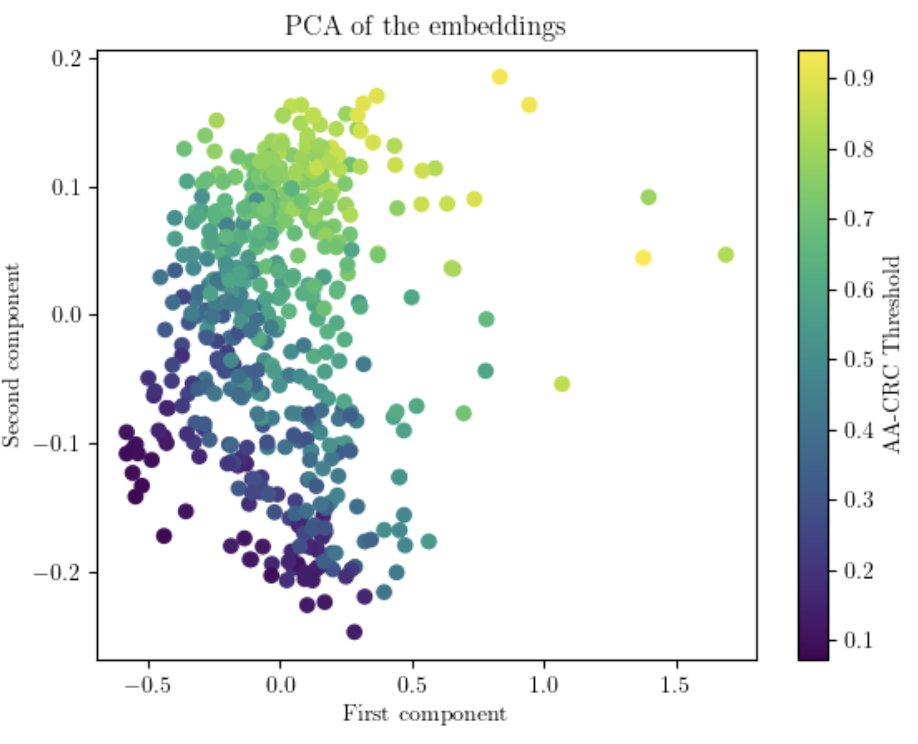}
    \end{subfigure}
    \hfill
    \begin{subfigure}[b]{0.4\textwidth}
        \centering
        \includegraphics[width=\textwidth]{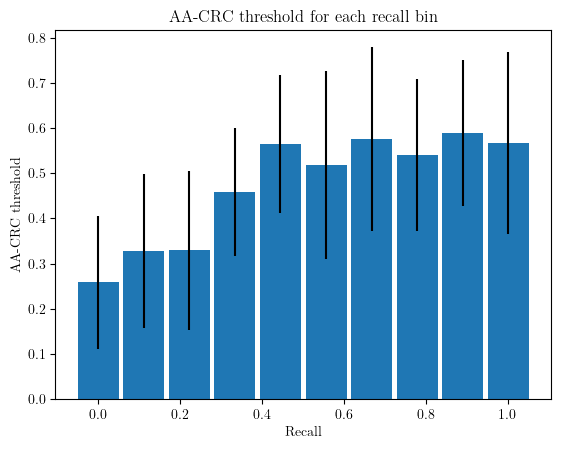}
    \end{subfigure}
    \caption{\revision{\textbf{Top figure.} Plot of the two first components of the PCA on the embeddings of the images. The color of each point correspond to the threshold returned by AA-CRC. \textbf{Bottom figure.} The right figure shows the within-group coverage for each of the adaptively selected groups. The red line is the target coverage level. 
    The coverage is almost exact for all groups for AA-CRC while almost all groups are either undercovered or overcovered this the standard CRC method}.}
    \label{fig:adaptiveness}
\end{figure}

\revision{Moreover, as shown in Figure~\ref{fig:adaptiveness}, when performing a PCA on the embeddings of the images and plot the two first components, and then attribute to each point the value of the threshold returned AA-CRC, we can visualize the adaptiveness of the methodology by seeing that the closer the embeddings, the closer the thresholds.}

\revision{We also evaluated the conditional performance of our method by analyzing the relationship between image recall at a fixed threshold of 0.5 and the adaptive thresholds assigned by the AA-CRC method. Recall at a fixed threshold reflects image difficulty: low recall indicates a challenging image requiring a lower threshold, while high recall suggests the threshold could be increased without degrading performance. To quantify this adaptiveness, we computed the Spearman correlation between recall at the fixed threshold and the AA-CRC-assigned thresholds. A positive correlation would indicate that AA-CRC assigns higher thresholds to easier images, demonstrating adaptive behavior. On the Polyp dataset, we observed a Spearman correlation of 0.41 with a p-value of $4\times10^{-23}$, confirming a positive relationship where higher recall generally corresponds to higher thresholds. To further illustrate this trend, we grouped images into recall bins from 0 to 1 in increments of 0.1, computing the average AA-CRC threshold and its standard deviation for each bin. A bar (Figure \ref{fig:adaptiveness} visualizes this relationship, showing an overall increase in thresholds as recall improves, reinforcing the adaptive nature of the AA-CRC method.}

\begin{figure*}[ht]
    \centering
    \begin{subfigure}[b]{\textwidth}
        \centering
        \includegraphics[width=.8\textwidth]{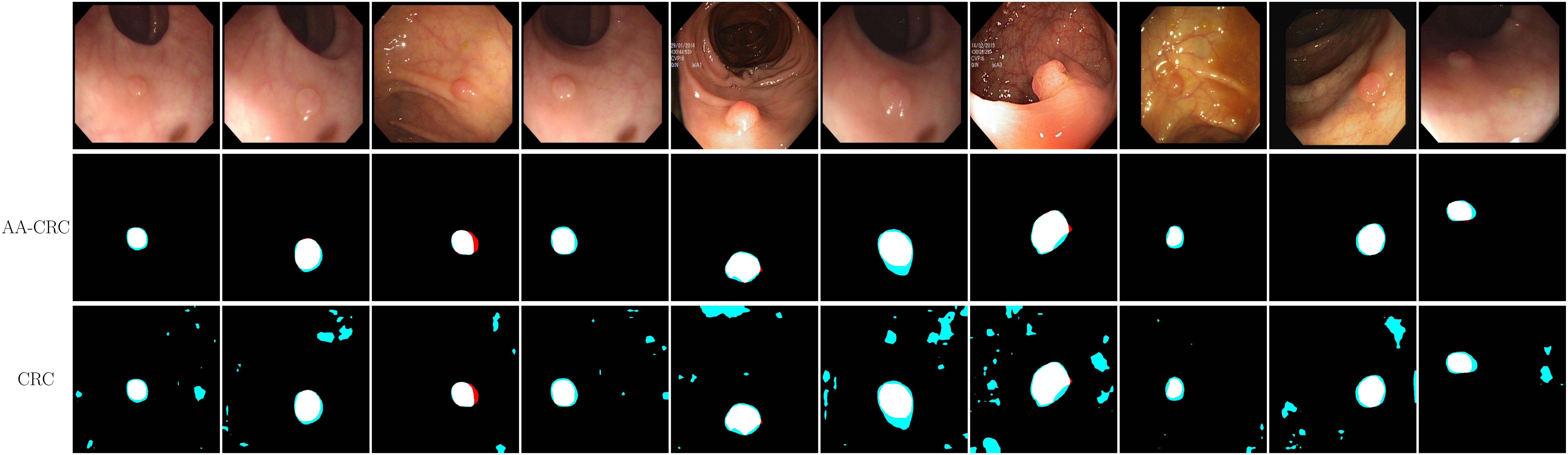}
    \end{subfigure}
    
    \begin{subfigure}[b]{0.3\textwidth}
        \centering
        \includegraphics[width=\textwidth]{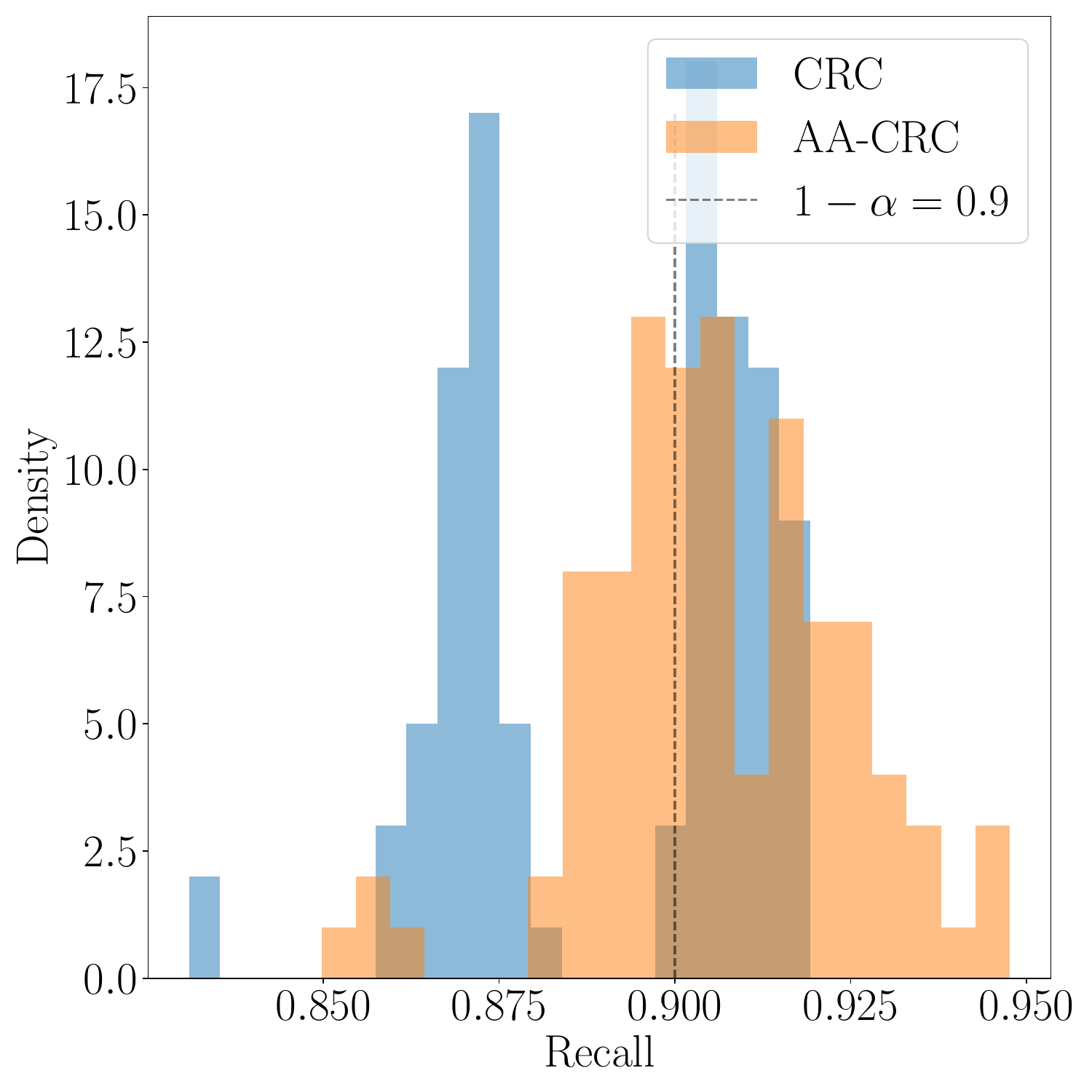}
    \end{subfigure}
    \begin{subfigure}[b]{0.3\textwidth}
        \centering
        \includegraphics[width=\textwidth]{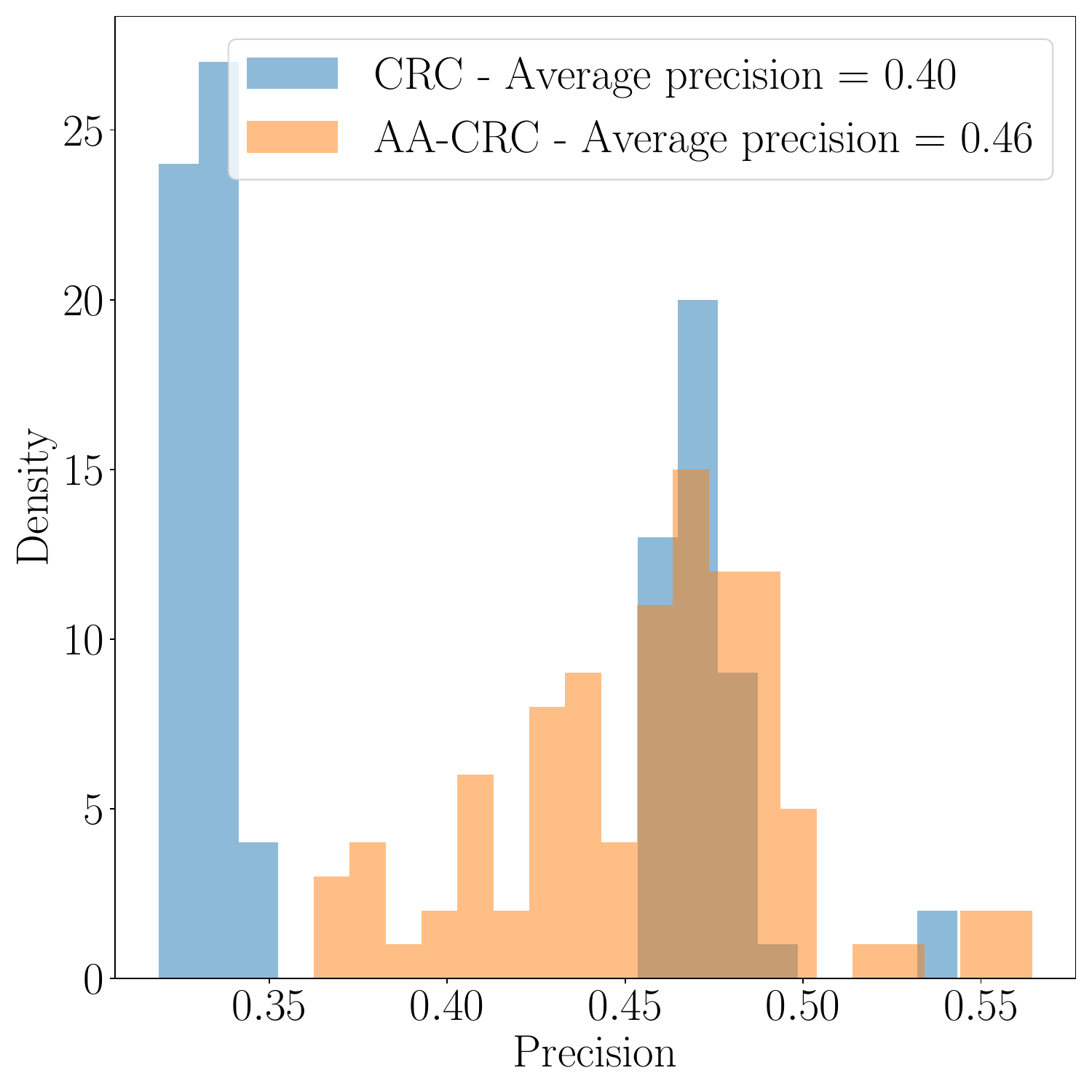}
    \end{subfigure}
    
    \caption{\textbf{Recall control for polyp segmentation.} The top figure compares the control of the recall made with our method (AA-CRC) to the control done with CRC. White pixels are true positives, blue pixels are false positives and red pixels are false negatives. The bottom figures represents the distribution of the recall of our procedure and distribution of the precision for both CRC and AA-CRC over 100 independent random data split.}
    \label{fig:main_polyp}
\end{figure*}

\paragraph{Fire segmentation dataset.}
We next perform experiments on fire segmentation from image data, using the dataset of~\citep{fire_dataset}.
We used a UNet~\citep{unet} for the segmentation backbone and a ResNet-50 to caculate the score embedding.
We chose the last hidden layer to have size 1024. 
We used 11671 images to train the UNet and ResNet-50 models, and respectively, 3432 and 6865 images for calibration and testing. 
To achieve better results in terms of precision, we performed a PCA~\citep{pca} on the embedding and added an intercept. 
The number of components was chosen such that explained variance ratio was equal to 0.85. 
Results are reported in Figure~\ref{fig:main_fire} with $\alpha=0.1$. 
The mean and standard deviation of recall over 100 random splits of the data are 0.898 and 0.003 respectively. 
The average precision of our method is 0.403, versus 0.363 for standard CRC, again improving the precision at the same recall level.

\begin{figure*}[!ht]
    \centering
    \centering
    \begin{subfigure}[b]{\textwidth}
        \centering
        \includegraphics[width=.8\textwidth]{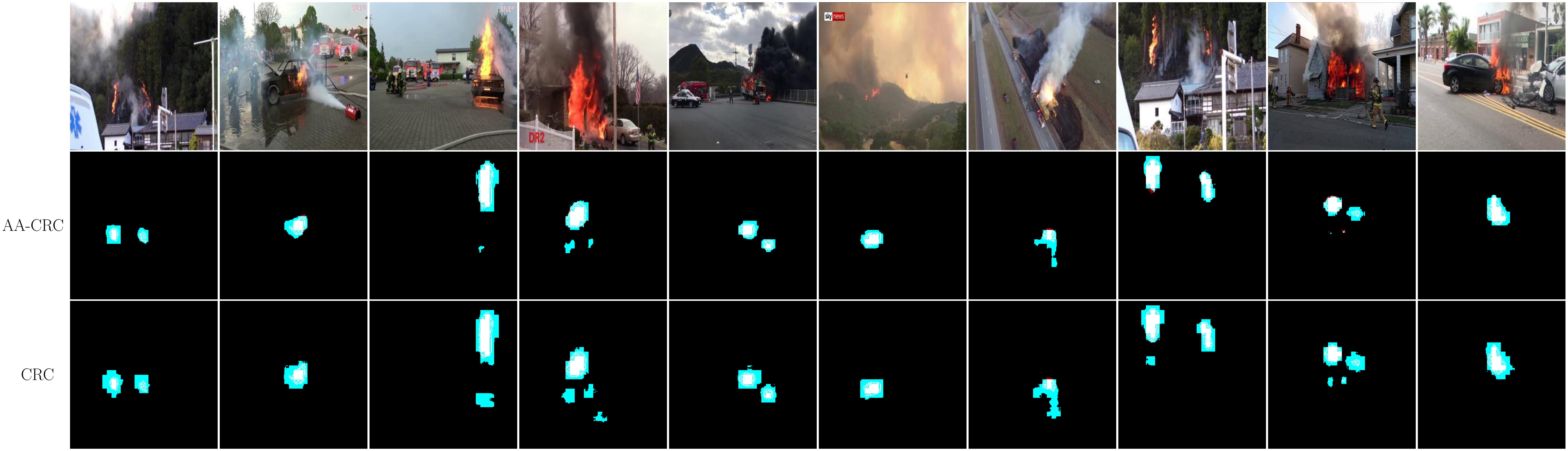}
    \end{subfigure}

    \begin{subfigure}[b]{0.3\textwidth}
        \centering
        \includegraphics[width=\textwidth]{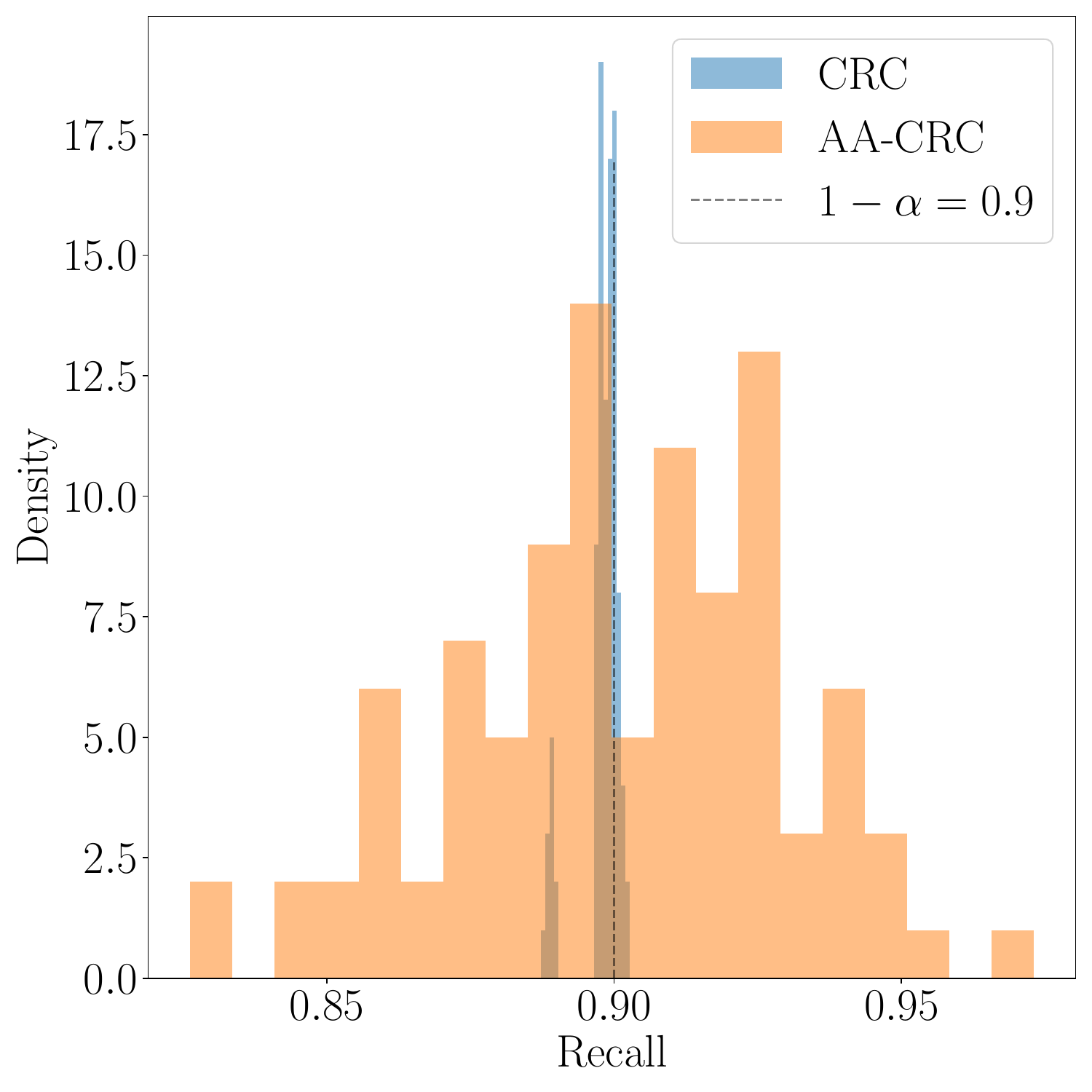}
    \end{subfigure}
    \begin{subfigure}[b]{0.3\textwidth}
        \centering
        \includegraphics[width=\textwidth]{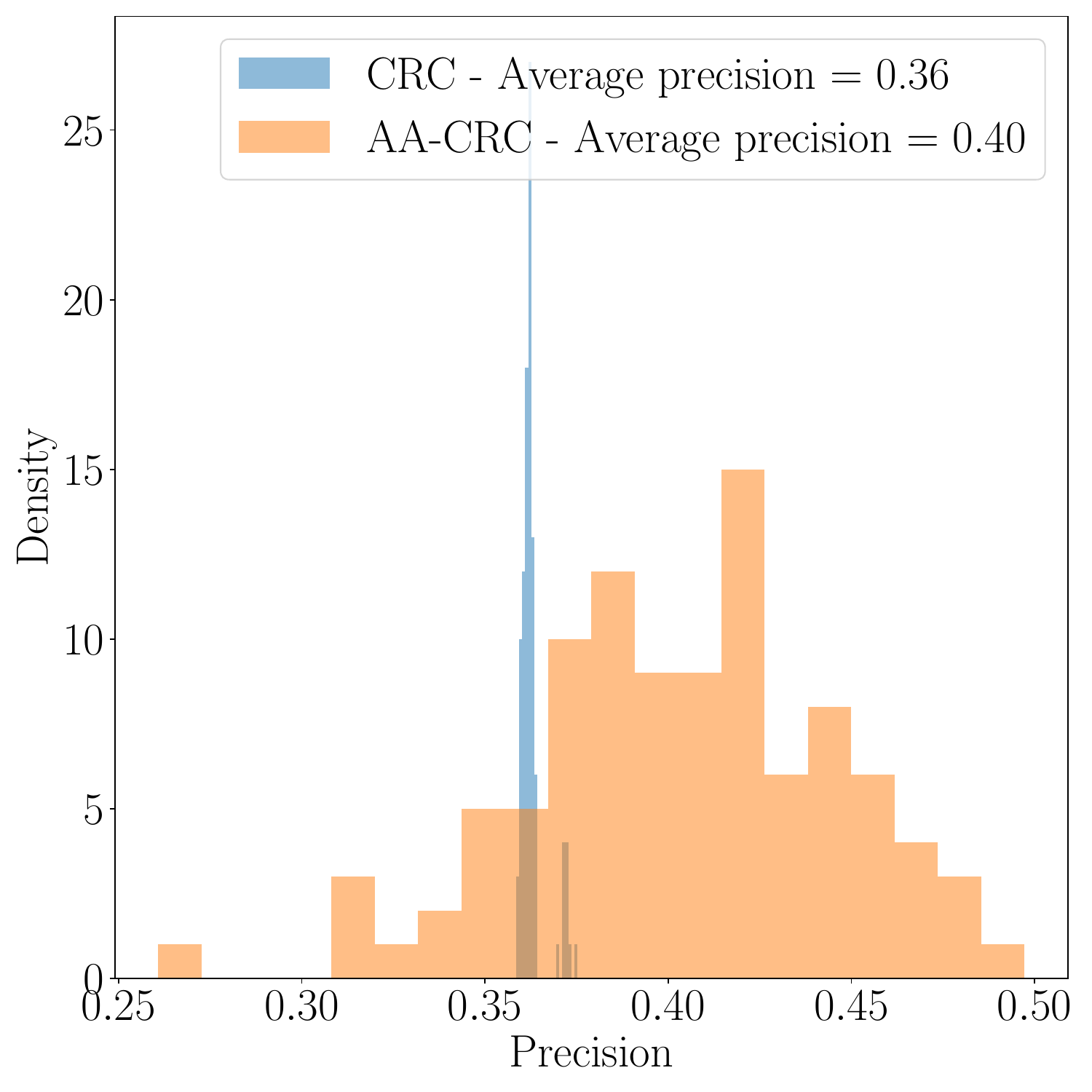}
    \end{subfigure}
    \caption{\textbf{Recall control for fire segmentation.} The top figure compares the control of the recall made with our method to the control done with CRC. White pixels are true positives, blues are false positives and reds are false negatives. The bottom figures represents the distribution of the recall of our procedure and distribution of the precision for both CRC and AA-CRC over 100 independent random data split.}
    \label{fig:main_fire}
\end{figure*}

\section{CONCLUSION AND FUTURE WORK}

We have presented a generalization of \cite{gibbs2023conformal}, AA-CRC that handles monotonic risks and adaptively chosen groups. 
We demonstrated the benefits of AA-CRC through the improvement of the precision in semantic segmentation tasks while controlling the recall. Moreover, we proposed a systematic methodology, for both tabular and image data, to construct adaptive function classes $\Lambda$ without needing any \emph{a priori} knowledge.
Future work will focus on the extension of this methodology to multiclass semantic segmentation, as well as using exploring additional choices of function spaces $\Lambda$.
Proving extensions of the remaining theorems in~\cite{gibbs2023conformal}---such as the bound in their (3.3), would also be available via standard analysis (e.g., via analyzing the same jump function as in~\cite{angelopoulos2024conformal}).
\clearpage
\newpage
\bibliographystyle{apalike}
\bibliography{bibliography}

\begin{thebibliography}{}

\bibitem[Aktaş, 2023]{fire_dataset}
Aktaş, M. (2023).
\newblock Fire {S}egmentation dataset - {K}aggle.

\bibitem[Amoukou and Brunel, 2023]{amoukou2023adaptive}
Amoukou, S.~I. and Brunel, N.~J. (2023).
\newblock Adaptive conformal prediction by reweighting nonconformity score.
\newblock {\em arXiv:2303.12695}.

\bibitem[Angelopoulos, 2024]{angelopoulosnote}
Angelopoulos, A.~N. (2024).
\newblock Note on full conformal risk control.

\bibitem[Angelopoulos et~al., 2024]{angelopoulos2024conformal}
Angelopoulos, A.~N., Bates, S., Fisch, A., Lei, L., and Schuster, T. (2024).
\newblock Conformal risk control.
\newblock In {\em 12th International Conference on Learning Representations}.

\bibitem[Angelopoulos and Tibshirani, 2023]{angelopoulos2023conformalpid}
Angelopoulos, Anastasios N.~andCand{\`e}s, E. and Tibshirani, R. (2023).
\newblock Conformal {PID} control for time series prediction.
\newblock In {\em Neural Information Processing Systems}.

\bibitem[Barber et~al., 2021]{barber2021limits}
Barber, Rina Foygel~andCand{\`e}s, E.~J., Ramdas, A., and Tibshirani, R.~J. (2021).
\newblock The limits of distribution-free conditional predictive inference.
\newblock {\em Information and Inference: A Journal of the IMA}, 10(2):455--482.

\bibitem[Bastani et~al., 2022]{bastani2022practical}
Bastani, O., Gupta, V., Jung, C., Noarov, G., Ramalingam, R., and Roth, A. (2022).
\newblock Practical adversarial multivalid conformal prediction.
\newblock {\em Advances in Neural Information Processing Systems}, 35:29362--29373.

\bibitem[Bernal et~al., 2017]{cvc}
Bernal, J., Tajkbaksh, N., S{\'a}nchez, F., Matuszewski, B., Chen, H., Yu, L., Angermann, Q., Romain, O., Rustad, B., Balasingham, I., Pogorelov, K., Choi, S., Debard, Q., Maier-Hein, L., Speidel, S., Stoyanov, D., Brandao, P., Cordova, H., S{\'a}nchez-Montes, C., Gurudu, S., Fern{\'a}ndez-Esparrach, G., Dray, X., Liang, J., and Histace, A. (2017).
\newblock Comparative validation of polyp detection methods in video colonoscopy: Results from the {MICCAI} 2015 endoscopic vision challenge.
\newblock {\em IEEE Transactions on Medical Imaging}, 99.

\bibitem[Breiman, 2001]{breiman2001random}
Breiman, L. (2001).
\newblock Random forests.
\newblock {\em Machine Learning}, 45(1):5--32.

\bibitem[Fan et~al., 2020]{fan2020pranet}
Fan, D.-P., Ji, G.-P., Zhou, T., Chen, G., Fu, H., Shen, J., and Shao, L. (2020).
\newblock Pranet: Parallel reverse attention network for polyp segmentation.
\newblock In {\em International Conference on Medical Image Computing and Computer-Assisted Intervention}, pages 263--273.

\bibitem[Gibbs et~al., 2023]{gibbs2023conformal}
Gibbs, I., Cherian, J.~J., and Cand{\`e}s, E.~J. (2023).
\newblock Conformal prediction with conditional guarantees.
\newblock {\em arXiv:2305.12616}.

\bibitem[He et~al., 2016]{resnet}
He, K., Zhang, X., Ren, S., and Sun, J. (2016).
\newblock Deep residual learning for image recognition.
\newblock In {\em IEEE Conference on Computer Vision and Pattern Recognition}, pages 770--778.

\bibitem[Jha et~al., 2020]{ksavir}
Jha, D., Smedsrud, P.~H., Riegler, M.~A., Halvorsen, P., De~Lange, T., Johansen, D., and Johansen, H.~D. (2020).
\newblock {Kvasir-SEG}: A segmented polyp dataset.
\newblock In {\em 26th International Conference on Multimedia Modeling}, pages 451--462.

\bibitem[Jin et~al., 2017]{jin2017escape}
Jin, C., Ge, R., Netrapalli, P., Kakade, S.~M., and Jordan, M.~I. (2017).
\newblock How to escape saddle points efficiently.
\newblock In {\em International Conference on Machine Learning}, pages 1724--1732. PMLR.

\bibitem[Jung et~al., 2021]{jung2021moment}
Jung, C., Lee, C., Pai, M., Roth, A., and Vohra, R. (2021).
\newblock Moment multicalibration for uncertainty estimation.
\newblock In {\em Conference on Learning Theory}, pages 2634--2678.

\bibitem[Jung et~al., 2022]{jung2022batch}
Jung, C., Noarov, G., Ramalingam, R., and Roth, A. (2022).
\newblock Batch multivalid conformal prediction.
\newblock {\em arXiv preprint arXiv:2209.15145}.

\bibitem[Kim et~al., 2019]{kim2019multiaccuracy}
Kim, M.~P., Ghorbani, A., and Zou, J. (2019).
\newblock Multiaccuracy: Black-box post-processing for fairness in classification.
\newblock In {\em AAAI/ACM Conference on AI, Ethics, and Society}, pages 247--254.

\bibitem[Koenker and Bassett~Jr, 1978]{koenker1978regression}
Koenker, R. and Bassett~Jr, G. (1978).
\newblock Regression quantiles.
\newblock {\em Econometrica: Journal of the {E}conometric {S}ociety}, 46(1):33--50.

\bibitem[Romano et~al., 2019]{cqr}
Romano, Y., Patterson, E., and Cand{\`e}s, E.~J. (2019).
\newblock Conformalized quantile regression.
\newblock In {\em Neural Information Processing Systems}.

\bibitem[Ronneberger et~al., 2015]{unet}
Ronneberger, O., Fischer, P., and Brox, T. (2015).
\newblock U-net: Convolutional networks for biomedical image segmentation.
\newblock In Navab, N., Hornegger, J., Wells, W.~M., and Frangi, A.~F., editors, {\em Medical Image Computing and Computer-Assisted Intervention}, pages 234--241.

\bibitem[Shwartz-Ziv and Armon, 2022]{bad_nn}
Shwartz-Ziv, R. and Armon, A. (2022).
\newblock Tabular data: Deep learning is not all you need.
\newblock {\em Information Fusion}, 81:84--90.

\bibitem[Silva et~al., 2014]{etis}
Silva, J., Histace, A., Romain, O., Dray, X., and Granado, B. (2014).
\newblock Toward embedded detection of polyps in wce images for early diagnosis of colorectal cancer.
\newblock {\em International Journal of Computer Assisted Radiology and Surgery}, 9(2):283--293.

\bibitem[Tajbakhsh et~al., 2015]{colondb}
Tajbakhsh, N., Gurudu, S.~R., and Liang, J. (2015).
\newblock Automated polyp detection in colonoscopy videos using shape and context information.
\newblock {\em IEEE Transactions on Medical Imaging}, 35(2):630--644.

\bibitem[Vazquez et~al., 2017]{cvc300}
Vazquez, D., Bernal, J., S{\'a}nchez, F.~J., Fern{\'a}ndez-Esparrach, G., Lopez, A.~M., Romero, A., Drozdzal, M., and Courville, A. (2017).
\newblock A benchmark for endoluminal scene segmentation of colonoscopy images.
\newblock {\em Journal of Healthcare Engineering}, pages 1--9.

\bibitem[Vovk, 2012]{vovk2012conditional}
Vovk, V. (2012).
\newblock Conditional validity of inductive conformal predictors.
\newblock In {\em {Asian Conference on Machine Learning}}, volume~25, pages 475--490.

\bibitem[Vovk et~al., 2005]{vovk2005algorithmic}
Vovk, V., Gammerman, A., and Shafer, G. (2005).
\newblock {\em {Algorithmic Learning in a Random World}}.
\newblock Springer.

\bibitem[Wold et~al., 1987]{pca}
Wold, S., Esbensen, K., and Geladi, P. (1987).
\newblock Principal component analysis.
\newblock {\em Chemometrics and Intelligent Laboratory Systems}, 2(1):37--52.
\newblock Multivariate Statistical Workshop for Geologists and Geochemists.

\bibitem[Zhang et~al., 2024]{zhang2024fair}
Zhang, L., Roth, A., and Zhang, L. (2024).
\newblock Fair risk control: A generalized framework for calibrating multi-group fairness risks.
\newblock {\em arXiv preprint arXiv:2405.02225}.

\end{thebibliography}

\end{document}